\documentclass[11pt]{article}
\usepackage[utf8]{inputenc} 
\usepackage{geometry}
\usepackage{graphicx}
\usepackage{booktabs}
\usepackage{array}
\usepackage{paralist}
\usepackage{verbatim}
\usepackage{subfig}
\usepackage{amsmath}
\usepackage{amssymb}
\usepackage{amsthm}
\usepackage{algorithmic}
\usepackage{algorithm}
\usepackage{cite}
\usepackage{fullpage}
\usepackage{algorithmic}
\usepackage{algorithm}
\floatname{algorithm}{Method}
\usepackage{color}

\newtheorem{proposition}{Proposition}[section]

\newtheorem{lemma}{Lemma}[section]

\newtheorem{theorem}{Theorem}[section]
\theoremstyle{remark}
\newtheorem{example}{Example}[section]

\newcommand{\rank}{\text{rank}}
\newcommand{\FSe}{\text{FS}_\varepsilon}

\newcommand{\by}{\mathbf{y}}
\newcommand{\bX}{\mathbf{X}}
\newcommand{\be}{\mathbf{e}}
\newcommand{\sgn}{\text{sgn}}

\newcommand{\B}{\boldsymbol}

\DeclareMathOperator*{\argmin}{arg\,min}
\DeclareMathOperator*{\argmax}{arg\,max}

\usepackage{fancyhdr}
\usepackage[nottoc,notlof,notlot]{tocbibind}
\usepackage[titles,subfigure]{tocloft}

\title{AdaBoost and Forward Stagewise Regression are First-Order Convex Optimization Methods}
\author{Robert M. Freund\thanks{MIT Sloan School of Management, 77 Massachusetts Avenue, Cambridge, MA   02139
({mailto:  rfreund@mit.edu}).  This author's research is supported by AFOSR Grant No. FA9550-11-1-0141 and the MIT-Chile-Pontificia Universidad Católica de Chile Seed Fund.}
\and Paul Grigas\thanks{MIT Operations Research Center, 77 Massachusetts Avenue, Cambridge, MA   02139
({mailto:  pgrigas@mit.edu}).  This author's research has been partially supported through an NSF Graduate Research Fellowship and the MIT-Chile-Pontificia Universidad Católica de Chile Seed Fund.}
\and Rahul Mazumder\thanks{MIT Operations Research Center, 77 Massachusetts Avenue, Cambridge, MA   02139
({mailto:  rahulmaz@mit.edu})}
}
\date{June 29, 2013}

\begin{document}
\maketitle

\begin{abstract}
Boosting methods are highly popular and effective supervised learning methods which combine weak learners into a single accurate model with good statistical performance.  In this paper, we analyze two well-known boosting methods, AdaBoost and Incremental Forward Stagewise Regression ($\FSe$), by establishing their precise connections to the Mirror Descent algorithm, which is a first-order method in convex optimization.  As a consequence of these connections we obtain novel computational guarantees for these boosting methods.  In particular, we characterize convergence bounds of AdaBoost, related to both the margin and log-exponential loss function, for any step-size sequence. Furthermore, this paper presents, for the first time, precise computational complexity results for $\FSe$.
\end{abstract}

\section{Introduction}
Boosting is a widely popular and successful supervised learning method which combines weak learners in a greedy fashion
to deliver accurate statistical models.
For an overview of the boosting approach, see, for example, Freund and Schapire~\cite{freund99short} and Schapire~\cite{schapire2003approach,schapire2012boosting}.  Though boosting (and in particular AdaBoost~\cite{freund99short})
was originally developed in the context of classification problems,
it is much more widely applicable~\cite{Friedman00greedyfunction}.
An important application of the boosting methodology in the
context of linear regression leads to Incremental Forward Stagewise Regression ($\FSe$)~\cite{HastieFSe,ESLBook,LARS}.
In this paper, we establish the equivalence of two boosting algorithms, AdaBoost and $\FSe$, to specific realizations of the Mirror Descent algorithm, which is a first-order method in convex optimization.
Through exact interpretations of these well-known boosting algorithms as specific first-order methods, we
leverage our understanding of computational complexity for first-order methods to
derive new computational guarantees for these algorithms. Such understanding of algorithmic computational complexity is also helpful from a statistical learning perspective, since
it enables one to derive bounds on
the number of  base models that need to be combined to get a ``reasonable" fit to the data.

\subsection*{Related Work, and Contributions}
We briefly outline some of the main developments in the complexity analysis of AdaBoost and
 $\FSe$ that appear to be closely related to the topic of this paper.\medskip

\noindent \textbf{AdaBoost and Optimization Perspectives:}  There has been a lot of interesting work connecting AdaBoost and related boosting methods to specific optimization problems and in understanding the computational guarantees of these methods with respect to these optimization problems.  In particular, much of the work has focused on two problems: maximizing the margin and minimizing the exponential loss. Mason et al.~\cite{boostingGradient} develop a general framework of boosting methods that correspond to coordinate gradient descent methods to minimize arbitrary loss functions, of which AdaBoost is a particular case with the exponential loss function.  For the problem of minimizing the exponential loss, Mukherjee et al.~\cite{rateAdaBoost2011} give precise convergence rates for the version of AdaBoost with step-sizes determined by a line-search, see also Telgarsky \cite{NIPS2011_0913}.  Schapire et al.~\cite{boostMargin97} show that the margin is inherently linked to the generalization error of the models produced by AdaBoost, thus it is highly desirable to maximize the margin in order to build predictive models. Several variants of AdaBoost have been developed specifically with this goal in mind, and these methods have been shown to converge to the maximum margin solution at appropriate rates (Ratsch and Warmuth~\cite{RatschWarmuthMargin2005}, Rudin et al.~\cite{RudinSmoothMargin2007}, Shalev-Shwartz and Singer~\cite{shaiWeakLinear2010}).\medskip

Under the assumption that the weak learner oracle returns the optimal base feature (also called weak hypothesis) for any distribution over the training data, we show herein that AdaBoost corresponds exactly to an instance of the Mirror Descent method \cite{NemirovskyYudin83, beckteb03mirror} for the primal/dual paired problem of edge minimization and margin maximization; the primal iterates $w^k$ are distributions over the examples and are attacking the edge minimization problem, and the dual iterates $\lambda^k$ are nonnegative combinations of classifiers that are attacking the maximum margin problem. In this minmax setting, the Mirror Descent method (and correspondingly AdaBoost) guarantees a certain bound on the duality gap $f(w^k) - p(\lambda^k)$ and hence on the optimality gap as a function of the step-size sequence, and for a simply chosen constant step-size the bound is $\sqrt{\frac{2\ln(m)}{k+1}}$.\medskip

In the case of separable data, we use a bound on the duality gap to directly infer a bound on the optimality gap for the problem of maximizing the margin.  We show precise rates of convergence for the optimal version of AdaBoost (without any modifications) with respect to the maximum margin problem for any given step-size rule. Our results seem apparently contradictory to Rudin et al.~\cite{rudinDynamics2004}, who show that even in the optimal case considered herein (where the weak learner always returns the best feature) AdaBoost may fail to converge to a maximum margin solution. However, in~\cite{rudinDynamics2004} their analysis is limited to the case where AdaBoost uses the originally prescribed step-size $\alpha_k := \frac{1}{2}\ln\left(\frac{1 + r_k}{1 - r_k}\right)$, where $r_k$ is the edge at iteration $k$, which can be interpreted as a line-search with respect to the exponential loss (not the margin) in the coordinate direction of the base feature chosen at iteration $k$ (see~\cite{boostingGradient} for a derivation of this). Our interpretation of AdaBoost in fact shows that the algorithm is structurally built to work on the maximum margin problem, and it is only the selection of the step-sizes that can cause convergence for this problem to fail.\medskip

In the case of non-separable data, a maximum margin solution is no longer informative; instead, we show that the edge $f(w^k)$ at iteration $k$ is exactly the $\ell_\infty$ norm of the gradient of the log-exponential loss evaluated at the current classifier, and we infer a bound on this norm through the bound on the duality gap. This bound quantifies the rate at which the classifiers produced by AdaBoost approach the first-order optimality condition for minimizing the log-exponential loss. Although precise objective function bounds on the optimality gap with respect to the exponential loss were given in \cite{rateAdaBoost2011}, their analysis is limited to the case of step-sizes determined by a line-search, as mentioned above. The step-sizes suggested by our Mirror Descent interpretation are quite different from those determined by a line-search, and furthermore although our bounds are specific to either the separable or non-separable case, the step-sizes we suggest do not depend on which case applies to a particular data set.\medskip

\noindent{\textbf{Forward Stagewise Regression and Optimization Perspectives:}}  The Incremental Forward Stagewise algorithm~\cite{HastieFSe, ESLBook,LARS} ($\FSe$) with shrinkage parameter $\varepsilon$ is a boosting algorithm for the linear regression problem that iteratively updates (by a small amount $\varepsilon$) the coefficient of the variable most correlated with the current residuals.
A principal reason behind why $\FSe$ is attractive from a statistical viewpoint is because of its ability to deliver \emph{regularized} solutions \eqref{sparsity1} by controlling the number of iterations $k$ along with the shrinkage parameter $\varepsilon$ with proper bias-variance tradeoff~\cite{ESLBook}.  The choice of the step-size plays an important role in the algorithm and has a bearing on the statistical properties of
the type of solutions produced.
For example, a step-size chosen by exact line-search on the least-squares loss function leads to the well known Forward Stagewise Algorithm---a greedy version of best subset selection~\cite{ESLBook}. \emph{Infinitesimal} Incremental Forward Stagewise Regression ($\text{FS}_0$, i.e., the limit of $\FSe$ as $\varepsilon \rightarrow 0+$) under some additional conditions on the data leads to a coefficient profile that is exactly the same as the LASSO solution path~\cite{HastieFSe,ESLBook}.
It is thus natural to ask what criterion might the $\FSe$ algorithm optimize?, and is it possible to
have computational complexity guarantees for
$\FSe$ --- and that can accommodate a flexible choice of steps-sizes?
To the best of our knowledge, a simple and complete answer to the above questions are heretofore unknown. In this paper, we answer these questions by showing that $\FSe$ is working towards minimizing the maximum correlation between the residuals and the predictors, which can also be interpreted as the $\ell_\infty$ norm of the gradient of the least-squares loss function. Our interpretation yields a precise bound on this quantity for any choice of the shrinkage parameter $\varepsilon$, in addition to the regularization/sparsity properties \eqref{sparsity1}.

\subsection{Notation}
For a vector $x \in \mathbb{R}^n$, $x_i$ denotes the $i^{\text{th}}$ coordinate; we use superscripts to index vectors in a sequence $\{x^k\}$. Let $e_j$ denote the $j^{\text{th}}$ unit vector in $\mathbb{R}^n$, $e = (1, \ldots, 1)$, and $\Delta_n = \{x \in \mathbb{R}^n : e^Tx = 1, x \geq 0\}$ is the $(n-1)$-dimensional unit simplex. Let $\|\cdot\|_q$ denote the $q$-norm for $q \in [1, \infty]$ with unit ball $B_q$, and let $\|v\|_0$ denote the number of non-zero coefficients of the vector $v$. For $A \in \mathbb{R}^{m \times n}$, let $\|A\|_{q_1, q_2} := \max\limits_{x : \|x\|_{q_1} \leq 1}\|Ax\|_{q_2}$ be the operator norm. For a given norm $\|\cdot\|$ on $\mathbb{R}^n$, $\|\cdot\|_\ast$ denotes the dual norm defined by $\|s\|_\ast = \max\limits_{x : \|x\| \leq 1}s^Tx$. Let $\partial f(\cdot)$ denote the subdifferential operator of a convex function $f(\cdot)$. The notation ``$\tilde v \leftarrow \argmax\limits_{v  \in S} \{f(v)\}$'' denotes assigning $\tilde v$ to be any optimal solution of the problem $\max\limits_{v  \in S} \{f(v)\}$.  For a convex set $P$ let $\Pi_P(\cdot)$ denote the Euclidean projection operator onto $P$, namely $\Pi_P(\bar x):=\argmin_{x \in P} \|x-\bar x\|_2$.

\section{Subgradient and Generalized Mirror Descent Methods: A Brief Review}\label{MirrorDescent}
Suppose we are interested in solving the following optimization problem:
\begin{equation}\label{poi1}
\mbox{(Primal):} \ \ \ \min\limits_{x \in P} \ f(x) \ ,
\end{equation}
where $P \subseteq \mathbb{R}^n$ is a closed convex set, $\mathbb{R}^n$ is considered with the given norm $\|\cdot\|$, and $f(\cdot) : P \to \mathbb{R}$ is a (possibly non-smooth) convex function.  Recall that $g$ is a subgradient of $f(\cdot)$ at $x$ if $f(y) \ge f(x) + g^T(y - x)$ for all $y \in P$, and we denote the set of subgradients of $f(\cdot)$ at $x$ by $\partial f(x)$.  We assume with no loss of generality that $\partial f(x) \ne \emptyset$ for all $x \in P$.  We presume that computation of a subgradient at $x \in P$ is not a burdensome task. Furthermore, we assume that $f(\cdot)$ has Lipschitz function values with Lipschitz constant $L_f$, i.e., we have $|f(x) - f(y)| \leq L_f\|x - y\|$ for all $x, y \in P$.\medskip

We are primarily interested in the case where $f(\cdot)$ is conveyed with minmax structure, namely:
\begin{equation}\label{eff}
f(x) := \max\limits_{\lambda \in Q} \ \phi(x, \lambda) \ ,
\end{equation}
where $Q \subseteq \mathbb{R}^m$ is a convex and compact set and $\phi(\cdot, \cdot)$ is a differentiable function that is convex in the first argument and concave in the second argument. In the case when $P$ is bounded, we define a dual function $p(\cdot) : Q \to \mathbb{R}$ by
\begin{equation}\label{ache}
p(\lambda) := \min\limits_{x \in P} \ \phi(x, \lambda) \ ,
\end{equation}
for which we may be interested in solving the dual problem:
\begin{equation}\label{poi2}
\mbox{(Dual):} \ \ \ \max\limits_{\lambda \in Q} \ p(\lambda) \ .
\end{equation}
Let $f^*$ denote the optimal value of \eqref{poi1}.  When $P$ is bounded let $p^*$ denote the optimal value of \eqref{poi2}, and the compactness of $P$ and $Q$ ensure that weak and strong duality hold:  $p(\lambda) \le p^* = f^* \leq f(x)$ for all $\lambda \in Q$ and $x \in P$.  The choice to call \eqref{poi1} the primal and \eqref{poi2} the dual is of course arbitrary, but this choice is relevant since the algorithms reviewed herein are \emph{not} symmetric in their treatment of the primal and dual computations.\medskip

The classical subgradient descent method for solving \eqref{poi1} determines the next iterate by taking a step $\alpha$ in the negative direction of a subgradient at the current point, and then projecting the resulting point back onto the set $P$.  If $x^k$ is the current iterate, subgradient descent proceeds by computing a subgradient $g^k \in \partial f(x^k)$, and determines the next iterate as $x^{k+1} \leftarrow \Pi_P (x^k -\alpha_k g^k)$, where $\alpha_k$ is the step-length, and $\Pi_P(\cdot)$ is the Euclidean projection onto the set $P$.\medskip

Note that in the case when $f(\cdot)$ has minmax structure \eqref{eff}, the ability to compute subgradients depends very much on the ability to solve the subproblem in the definition \eqref{eff}.  Indeed, \begin{equation}\label{subdiff} \mbox{if} \ \  \tilde{\lambda}^k \in \argmax_{\lambda \in Q} \phi(x^k, \lambda) \ \ , \ \mbox{then} \ \ g^k \leftarrow \nabla_x \phi(x^k, \tilde \lambda^k) \in \partial f(x^k) \ , \end{equation} that is, $g^k$ is a subgradient of $f(\cdot)$ at $x^k$.  This fact is very easy to derive, and is a special case of the more general result known as Danskin's Theorem, see \cite{bertsekas}.\medskip

In consideration of the computation of the subgradient \eqref{subdiff} for problems with minmax structure \eqref{eff}, the formal statement of the subgradient descent method is presented in Algorithm \ref{subgraddescent}.\medskip

\begin{algorithm}
\caption{Subgradient Descent Method (for problems with minmax structure)}\label{subgraddescent}
\begin{algorithmic}
\STATE Initialize at $x^0 \in P$, $k \leftarrow 0$\medskip

At iteration $k$:
\STATE 1. Compute:
\begin{description}
\item $\tilde{\lambda}^k \gets \argmax\limits_{\lambda \in Q} \ \phi(x^k, \lambda)$
\item $g^k \gets \nabla_x \phi(x^k, \tilde{\lambda}^k)$
\end{description}
\STATE 2. Choose $\alpha_k \geq 0$ and set:
\begin{description}
\item $x^{k+1} \gets \Pi_P(x^k - \alpha_k g^k)$
\end{description}
\end{algorithmic}
\end{algorithm}

The Mirror Descent method \cite{NemirovskyYudin83, beckteb03mirror} is a generalization of the subgradient descent method.  The Mirror Descent method requires the selection of a differentiable ``$1$-strongly convex'' function $d(\cdot) : P \to \mathbb{R}$ which is defined to be a function with the following (strong) convexity property:
\begin{equation*}
d(x) \geq d(y) + \nabla d(y)^T(x - y) + \frac{1}{2}\|x - y\|^2 \text{ for all }x, y \in P \ .
\end{equation*}
The function $d(\cdot)$ is typically called the ``prox function.''  The given prox function $d(\cdot)$ is also used to define a distance function:
\begin{equation}\label{bregman}
D(x,y) := d(x) - d(y) -\nabla d(y)^T(x - y) \geq \frac{1}{2}\|x - y\|^2 \text{ for all }x, y \in P \ .
\end{equation}
One can think of $D(x,y)$ as a not-necessarily-symmetric generalization of a distance metric (induced by a norm), in that $D(x,y) \ge \frac{1}{2}\|x-y\|^2 \ge 0$, $D(x,y) = 0$ if and only if $x=y$, but it is not generally true (nor is it useful) that $D(x,y) = D(y,x)$.  $D(x,y)$ is called the Bregman function or the Bregman distance.  With these objects in place, the Mirror Descent (proximal subgradient) method for solving (\ref{poi1}) is presented in Algorithm \ref{standardprox}.\medskip

\begin{algorithm}
\caption{Mirror Descent Method (applied to problems with minmax structure)}\label{standardprox}
\begin{algorithmic}
\STATE Initialize at $x^0 \in P$, $\lambda^0 = 0, k = 0$\medskip

At iteration $k$:
\STATE 1. Compute:
\begin{description}
\item $\tilde{\lambda}^k \gets \argmax\limits_{\lambda \in Q} \ \phi(x^k, \lambda)$
\item $g^k \gets \nabla_x \phi(x^k, \tilde{\lambda}^k)$
\end{description}
\STATE 2. Choose $\alpha_k \geq 0$ and set:
\begin{description}
\item $x^{k+1} \gets \arg\min\limits_{x \in P}\left\{\alpha_k (g^k)^Tx + D(x,x^k)\right\}$
\item $ \ $
\item $\lambda^{k+1} \gets \frac{\sum_{i = 0}^k\alpha_i\tilde{\lambda}^i}{\sum_{i = 0}^k\alpha_i}$
\end{description}
\end{algorithmic}
\end{algorithm}

The sequence $\{\lambda^k\}$ constructed in the last line of Step (2.) of Mirror Descent plays no role in the actual dynamics of Algorithm \ref{standardprox} and so could be ignored; however $\lambda^k$ is a feasible solution to the dual problem \eqref{poi2} and we will see that  the sequence $\{\lambda^k\}$ has precise computational guarantees with respect to problem \eqref{poi2}.  The construction of $x^{k+1}$ in Step (2.) of Mirror Descent involves the solution of an optimization subproblem; the prox function $d(\cdot)$ should be chosen so that this subproblem can be easily solved, i.e., in closed form or with a very efficient algorithm.\medskip

Note that the subgradient descent method described in Algorithm \ref{subgraddescent} is a special case of Mirror Descent using the ``Euclidean'' prox function $d(x):= \frac{1}{2}\|x\|_2^2$.  With this choice of prox function, Step (2.) of Algorithm \ref{standardprox} becomes:
\begin{equation*}
x^{k+1} \gets \arg\min\limits_{x \in P}\left\{\left(\alpha_k g^k - x^k\right)^Tx + \frac{1}{2}x^Tx \right\}= \Pi_P(x^k - \alpha_kg^k) \ ,
\end{equation*}
(since $D(x,x^k)=(-x^k)^Tx + \frac{1}{2}x^Tx + \frac{1}{2}\|x^k\|_2^2$), and is precisely the subgradient descent method with step-size sequence $\{\alpha_k\}$.  Indeed, the sequence $\{\alpha_k\}$ in the Mirror Descent method is called the ``step-size'' sequence in light of the analogy to subgradient descent. Below we present an example of a version of Mirror Descent with a prox function that is not Euclidean, which will be useful in the analysis of the algorithm AdaBoost.\medskip

\begin{example}{\bf{Multiplicative Weight Updates for Optimization on the Standard Simplex in $\mathbb{R}^n$}}\label{mwexample}\\
Consider optimization of $f(x)$ on $P = \Delta_n := \{x \in \mathbb{R}^n : e^Tx = 1, x \geq 0\}$, the standard simplex in $\mathbb{R}^n$, and let $d(x) = e(x) := \sum_{i=1}^n x_i\ln(x_i) + \ln(n)$ be the entropy function. It is well-known that $e(\cdot)$ is a $1$-strongly convex function on $\Delta_n$ with respect to the $\ell_1$ norm, see for example \cite{nest05smoothing} for a short proof.  Given any $c \in \mathbb{R}^n$, it is straightforward to verify that the optimal solution $\bar{x}$ of a problem of format $\min\limits_{x \in P}\left\{c^Tx + d(x)\right\}$ is given by:
\begin{equation}\label{entropysolution}
\bar{x}_i = \frac{\exp(-c_i)}{\sum_{l = 1}^n\exp(-c_l)} \ \ \ i = 1, \ldots, n \ .
\end{equation}
Using the entropy prox function, it follows that for each $i \in \{1, \ldots, n\}$, the update of $x^k$ in Step (2.) of Algorithm \ref{standardprox} assigns:
\begin{equation*}
x^{k+1}_i \propto \exp(-(\alpha_k g^k - \nabla e(x^k))_i) = \exp(1 + \ln(x^k_i) - \alpha_kg^k_i) \propto x^k_i \cdot \exp(-\alpha_k g^k_i) \ ,
\end{equation*}
which is an instance of the multiplicative weights update rule \cite{AroraHK12}.
\end{example}\medskip

We now state two well-known complexity bounds for the Mirror Descent method (Algorithm \ref{standardprox}), see for example \cite{beckteb03mirror}.  In the general case we present a bound on the optimality gap of the sequence $\{x^k\}$ for the primal problem \eqref{poi1} that applies for any step-size sequence $\{\alpha_k\}$, and in the case when $P$ is compact we present a similar bound on the duality gap of the sequences $\{x^k\}$ and $\{\lambda^k\}$.  Both bounds can be specified to $O\left(\frac{1}{\sqrt{k}}\right)$ rates for particularly chosen step-sizes.\medskip

\begin{theorem}{\bf{(Complexity of Mirror Descent}\cite{beckteb03mirror, polyak, nesterovBook})}\label{proxcomplexity}
Let $\{x^k\}$ and $\{\lambda^k\}$ be generated according to the Mirror Descent method (Algorithm \ref{standardprox}). Then for each $k \geq 0$ and for any $x \in P$, the following inequality holds:
\begin{equation}\label{mirror_bound1}
\min_{i \in \{0,\ldots,k\}} f(x^i) - f(x) \ \ \leq \ \  \frac{D(x,x^0) + \frac{1}{2}L_f^2\sum_{i = 0}^k\alpha_i^2}{\sum_{i=0}^k\alpha_i} \ .
\end{equation}
If $P$ is compact and $\bar{D} \geq \max\limits_{x \in P}D(x,x^0)$, then for each $k \geq 0$ the following inequality holds:
\begin{equation}\label{mirror_bound2}
\min_{i \in \{0,\ldots,k\}} f(x^i) -p(\lambda^{k+1})  \ \ \leq \ \ \frac{\bar{D}+ \frac{1}{2}L_f^2\sum_{i = 0}^k\alpha_i^2}{\sum_{i=0}^k\alpha_i} \ .
\end{equation}
\end{theorem}
These bounds are quite general; one can deduce specific bounds, for example, by specifying a step-size sequence $\{\alpha^k\}$, a prox function $d(\cdot)$, a value of $x$ in \eqref{mirror_bound1} such as $x=x^*$, etc., see Propositions \ref{prop1} and \ref{prop2} where these specifications are illustrated in the case when $P$ is compact, for example.

\begin{proposition}\label{prop1} Suppose we {\em a priori} fix the number of iterations $k$ of Algorithm \ref{standardprox} and use a constant step-size sequence:
\begin{equation}\label{constant-rule}\alpha_i = \bar \alpha = \frac{1}{L_f}\sqrt{\frac{2\bar{D}}{k+1}} \  \ \ \ \ \mathrm{for} \ \ i =0, \ldots, k \ . \end{equation} Then
\begin{equation}\label{opt-alpha-bound}
\min_{i \in \{0,\ldots,k\}}f(x^i) - p(\lambda^{k+1}) \ \  \leq \ \ L_f\sqrt{\frac{2\bar{D}}{k+1}} \ .
\end{equation}\end{proposition}
\begin{proof} This follows immediately from \eqref{mirror_bound2} by substituting in \eqref{constant-rule} and rearranging terms.
\end{proof}
Indeed, the bound \eqref{opt-alpha-bound} is in fact the best possible bound for a generic subgradient method, see~\cite{NemirovskyYudin83}.

\begin{proposition}\label{prop2} Suppose we use the dynamic step-size sequence:
\begin{equation}\label{dynamic-rule}\alpha_i := \frac{1}{L_f}\sqrt{\frac{2\bar{D}}{i+1}} \ \ \ \ \  \mathrm{for} \ \ i \geq 0 \ . \end{equation} Then after $k$ iterations the following holds:
\begin{equation}\label{dynamic-bound}
\min_{i \in \{0,\ldots,k\}}f(x^i) - p(\lambda^{k+1}) \ \ \leq \ \  \frac{L_f\sqrt{\frac{1}{2}\bar{D}}\left(2 + \ln(k+1)\right)}{2(\sqrt{k+2} - 1)} \ = \ O\left(\frac{L_f \sqrt{\bar D}\ln(k)}{\sqrt{k}}\right)\ .
\end{equation}\end{proposition}
\begin{proof} Substituting \eqref{dynamic-rule} in \eqref{mirror_bound2} and rearranging yields:
$$\min_{i \in \{0,\ldots,k\}}f(x^i) - p(\lambda^{k+1}) \ \ \leq \ \  \frac{L_f\sqrt{\frac{1}{2}\bar{D}}\left(1 + \sum_{i=0}^{k}\frac{1}{i+1} \right)}{\sum_{i=0}^{k}\frac{1}{\sqrt{i+1}}} \ . $$
The proof is completed by using the integral bounds $$1 + \sum_{i=0}^{k}\frac{1}{i+1} \le 2 + \int_1^{k+1} \frac{1}{t}dt = 2+\ln(k+1) \ , $$ and
$$\sum_{i=0}^{k}\frac{1}{\sqrt{i+1}} = \sum_{i=1}^{k+1}\frac{1}{\sqrt{i}} \ge \int_1^{k+2} \frac{1}{\sqrt{t}}dt = 2\sqrt{k+2} - 2 \ . $$
\end{proof}

Finally, consider the subgradient descent method (Algorithm \ref{subgraddescent}), which is Mirror Descent using $d(x)=\frac{1}{2}\|x\|_2^2$ in the case when the optimal value $f^\ast$ of \eqref{poi1} is known. Suppose that the step-sizes are given by $\alpha_k := \frac{f(x^k) - f^\ast}{\|g^k\|_2^2}$, then it is shown in Polyak \cite{polyak} that for any optimal solution $x^*$ of \eqref{poi1} it holds that:
\begin{equation}\label{subgrad_bound2}
\min_{i \in \{0,\ldots,k\}} f(x^i) - f^\ast \ \ \leq \ \  \frac{L_f\|x^0 - x^\ast\|_2}{\sqrt{k+1}} \ .
\end{equation}

\section{AdaBoost as Mirror Descent}\label{adaboostMirror}
We are given a set of base classifiers (also called weak hypotheses) $\mathcal{H} = \{h_1, \ldots, h_n\}$ where each $h_j : \mathcal{X} \to \{-1, 1\}$, and we are given training data (examples) $(x_1, y_1), \ldots, (x_m, y_m)$ where each $x_i \in \mathcal{X}$ ($\mathcal{X}$ is some measurement space) and each $y_i \in \{-1, +1\}$.\footnote{Actually our results also hold for the more general confidence-rated classification setting, where $h_j : \mathcal{X} \to [-1, 1]$ and $y_i \in [-1,1]$.}  We have access to a weak learner $\mathcal{W}(\cdot) : \Delta_m \to \{1, \ldots, n\}$ that, for any distribution $w$ on the examples ($w \in \Delta_m)$, returns an index $j^*$ of a base classifier $h_{j^*}$ in $\mathcal{H}$ that does best on the weighted example determined by $w$.  That is, the weak learner $\mathcal{W}(w)$ computes $j^* \in \argmax\limits_{j\in \{1, \ldots, n\}}\sum_{i = 1}^mw_iy_ih_j(x_i)$ and we write `` $j^* \in {\cal W}(w)$ '' in a slight abuse of notation.  Even though $n$ may be extremely large, we assume that it is easy to compute an index $j^* \in \mathcal{W}(w)$ for any $w \in \Delta_m$. Algorithm \ref{adaboost} is the algorithm AdaBoost, which constructs a sequence of distributions $\{w^k\}$ and a sequence $\{H_k\}$ of nonnegative combinations of base classifiers with the intent of designing a classifier $\text{sign}(H_k)$ that performs significantly better than any base classifier in $\mathcal{H}$.
\floatname{algorithm}{Algorithm}
\begin{algorithm}
\caption{AdaBoost}\label{adaboost}
\begin{algorithmic}
\STATE Initialize at $w^0 = (1/m, \ldots, 1/m), H_{0} = 0, k = 0$\medskip

At iteration $k$:
\STATE 1. Compute $j_k \in \mathcal{W}(w^k)$\medskip
\STATE 2. Choose $\alpha_k \geq 0$ and set:
\begin{description}
\item $H_{k+1} \gets H_{k} + \alpha_kh_{j_k}$
\item $w^{k+1}_i \gets w^k_i\exp(-\alpha_ky_ih_{j_k}(x_i)) \ \ i=1, \ldots, m$, and re-normalize $w^{k+1}$ so that $e^Tw^{k+1} = 1$
\end{description}
\end{algorithmic}
\end{algorithm}
\floatname{algorithm}{Method}

Notice that AdaBoost maintains a sequence of classifiers $\{H_k\}$ that are (nonnegative) linear combinations of base classifiers in $\mathcal{H}$.
Strictly speaking, a linear combination $H = \sum_{j=1}^n\lambda_jh_{j}$ of base classifiers in $\mathcal{H}$ is a function from $\mathcal{X}$ into the reals, and the classifier determined by $H$ is $\text{sign}(H)$; however, for simplicity we will refer to the linear combination $H$ as a classifier, and we say that the coefficient vector $\lambda \in \mathbb{R}^n$ determines the classifier $H$.

\subsection{AdaBoost is a Specific Case of Mirror Descent}\label{ada-subsect}
Here we show that AdaBoost corresponds to a particular instance of the Mirror Descent method (Algorithm \ref{standardprox}) applied to the particular primal problem of minimizing the {\em edge} in the space of ``primal'' variables $w \in \Delta_m$ which are distributions over the training data; and through duality, maximizing the {\em margin} in the space of ``dual'' variables of normalized classifiers represented by vectors $\lambda \in \mathbb{R}^n$ of coefficients which determine classifiers $\sum_{j=1 }^n\lambda_jh_j$.  We also show that the edge of $w^k$ is exactly the $\ell_\infty$ norm of the gradient of the log-exponential loss function.  Utilizing the computational complexity results for Mirror Descent (Theorem \ref{proxcomplexity}), we then establish  guarantees on the duality gap for these duality paired problems.  When the data are separable, these guarantees imply that the sequence of classifiers $\{H_k\}$ constructed in AdaBoost are in fact working on the problem of maximizing the \emph{margin}, with specific computational guarantees thereof for any step-size sequence $\{\alpha_k \}$.  When the data is not separable, these guarantees imply that the classifiers $\{H_k\}$ are in fact working on the problem of driving the $\ell_\infty$ norm of the gradient of the log-exponential loss function to zero, with specific computational guarantees thereof for any step-size sequence $\{\alpha_k \}$.  Let us see how this works out.\medskip

For convenience define the feature matrix $A \in \mathbb{R}^{m \times n}$ componentwise by $A_{ij} := y_i h_j(x_i)$, and let $A_j$ denote the $j$th column of $A$, and define $\phi(w,\lambda)=w^TA\lambda$ where we use $w$ instead of $x$ to represent the primal variable.  For any distribution $w \in \Delta_m$, $w^TA_j$ is the {\em edge} of classifier $h_j$ with respect to $w$, and \begin{equation}\label{poly}
f(w) :=  \max\limits_{j \in \{1, \ldots, n\}} \ w^TA_j = \max\limits_{\lambda \in \Delta_n} \ w^TA\lambda = \max\limits_{\lambda \in \Delta_n} \ \phi(w,\lambda)
\end{equation} is the maximum edge over all base classifiers, and we call $f(w)$ the edge with respect to $w$.  The optimization problem of minimizing the edge over all distributions $w$ is:
\begin{equation}\label{adapoi2}
\mbox{(Primal):} \ \ \ \min_{w \in \Delta_m} \ f(w) \ .
\end{equation}
Here \eqref{poly} and \eqref{adapoi2} are precisely in the format of \eqref{eff} and \eqref{poi1} with $P=\Delta_m$ and $Q=\Delta_n$.  We can construct the dual of the edge minimization problem following \eqref{ache} and \eqref{poi2}, whereby we see that the dual function is:
\begin{equation}\label{extra}
p(\lambda) :=  \min_{w \in \Delta_m} \ \phi(w,\lambda) =  \min_{w \in \Delta_m} \ w^TA\lambda = \min\limits_{i\in \{1, \ldots, m\}}(A\lambda)_i \ ,
\end{equation} and the dual problem is:
\begin{equation}\label{adapoi1}
\mbox{(Dual):} \ \ \ \max\limits_{\lambda \in \Delta_n} \ p(\lambda) \ .
\end{equation}
The {\em margin} achieved by the $\lambda$ on example $i$ is $(A\lambda)_i$, whereby $p(\lambda)$ is the least margin achieved by $\lambda$ over all examples, and is simply referred to as the margin of $\lambda$.  Because $p(\lambda)$ is positively homogeneous ($p(\beta \lambda) = \beta p(\lambda)$ for $\beta \ge 0$), it makes sense to normalize $\lambda$ when measuring the margin, which we do by re-scaling $\lambda$ so that $\lambda \in \Delta_n$.  Therefore the dual problem is that of maximizing the margin over all normalized nonnegative classifiers.  Note also that it is without loss of generality that we assume $\lambda \geq 0$ since for any base classifier $h_j \in \mathcal{H}$ we may add the classifier $-h_j$ to the set $\mathcal{H}$ if necessary.  Consider the classifier $H_k$ constructed in Step (2.) of AdaBoost.  It follows inductively that ${H}_k = \sum_{i = 0}^{k-1}\alpha_i h_{j_i}$, and we define the normalization of $H_k$ as:
\begin{equation}\label{normalized}
\bar{H}_k := \frac{H_k}{\sum_{i=0}^{k-1}\alpha_i} = \frac{\sum_{i = 0}^{k-1}\alpha_ih_{j_i}}{\sum_{i=0}^{k-1}\alpha_i}  \ .
\end{equation}
In addition to the margin function $p(\lambda)$, it will be useful to look at log-exponential loss function $L(\cdot) : \mathbb{R}^n \to \mathbb{R}$ which is defined as:
\begin{equation}\label{exp-loss}
L(\lambda) = \log\left(\frac{1}{m}\sum_{i = 1}^m\exp\left(-(A\lambda)_i\right)\right) \ .
\end{equation}  It is well-known that $L(\cdot)$ and $p(\cdot)$ are related by:  $-p(\lambda) - \ln(m) \le L(\lambda) \le -p(\lambda)$ for any $\lambda$.\medskip

We establish the following equivalence result.\medskip

\begin{theorem}\label{adaBoost-equiv}
The sequence of weight vectors $\{w^k\}$ in AdaBoost arise as the sequence of primal variables in Mirror Descent applied to the primal problem \eqref{adapoi2}, using the entropy prox function $d(w):=e(w)=\sum_{i=1}^m w_i\ln(w_i) + \ln(m)$, with step-size sequence $\{\alpha_k\}$ and initialized at $w^0 = (1/m, \ldots, 1/m)$.  Furthermore, the sequence of normalized classifiers $\{\bar{H}_k\}$ produced by AdaBoost arise as the sequence of dual variables $\{\lambda^k\}$ in Mirror Descent, and the margin of the classifier $\bar{H}_k$ is $p(\lambda^k)$.
\end{theorem}

\begin{proof}
By definition of the weak learner and \eqref{poly} combined with \eqref{subdiff}, we have for any $w \in \Delta_m$
\begin{equation*}
j^\ast \in \mathcal{W}(w) \Longleftrightarrow j^\ast \in \argmax_{j\in \{1, \ldots, n\}} w^TA_j \Longleftrightarrow e_{j^\ast} \in \argmax_{\lambda \in \Delta_n}w^TA\lambda \Longleftrightarrow A_{j^\ast} \in \partial f(w) \ ,
\end{equation*}
whereby Step (1.) of AdaBoost is identifying a vector $g^k := A_{j_k} \in \partial f(w^k)$. Moreover, since $g^k_i = y_ih_{j_k}(x_i)= A_{i,j_k}$, the construction of $w^{k+1}$ in Step (2.) of AdaBoost is exactly setting $w^{k+1} \gets \arg\min\limits_{w \in \Delta_m}\left\{\alpha_k(g^k)^Tx + D(w, w^k)\right\}$ (where $D(\cdot, \cdot)$ is the Bregman distance function arising from the entropy function), as discussed in Example \ref{mwexample}.  Therefore the sequence $\{w^k\}$ is a sequence of primal variables in Mirror Descent with the entropy prox function.  Also notice from Step (1.) of AdaBoost and the output of the weak learner ${\cal W}(w^k)$ that $e_{j_k} \in \argmax\limits_{\lambda \in \Delta_n} \ (w^k)^TA\lambda$, which gives the correspondence $\tilde{\lambda}^k = e_{j_k}$ at Step (1.) of Mirror Descent.  Let $\{\lambda^k\}$ denote the corresponding sequence of dual variables defined in Step (2.) of Mirror Descent; it therefore follows that:
\begin{equation*}
\lambda^k :=  \frac{\sum_{i=0}^{k-1}\alpha_i \tilde{\lambda}^i}{\sum_{i=0}^{k-1}\alpha_i} = \frac{\sum_{i=0}^{k-1}\alpha_i e_{j_i}}{\sum_{i=0}^{k-1}\alpha_i} \ ,
\end{equation*}
whereby $\bar{H}_k$ defined in \eqref{normalized} is precisely the classifier determined by $\lambda^k$, and it follows that the margin of $\bar{H}_k$ is $p(\lambda^k)$.
\end{proof}

Let $\{\hat{\lambda}^k\}$ denote the sequence of coefficient vectors of the un-normalized classifiers $\{H_k\}$ produced by AdaBoost, where $\hat{\lambda}^k = \sum_{i=0}^{k-1}\alpha_i e_{j_i}$.  We also have the following relationship concerning the norm of the gradient of log-exponential loss function.\medskip

\begin{lemma}\label{edge-lemma}
For every iteration $k \geq 0$ of AdaBoost, the edge $f(w^k)$ and the un-normalized classifier $H_k$ with coefficient vector $\hat{\lambda}^k$ satisfy:
\begin{equation*}
f(w^k) = \|\nabla L(\hat{\lambda}^k)\|_\infty \ .
\end{equation*}
\end{lemma}

\begin{proof}
By our assumption that the set of base classifiers $\mathcal{H}$ is closed under negation, we have for any $w$ that $f(w) = \max\limits_{\lambda \in \Delta_n}w^TA\lambda = \max\limits_{\lambda : \|\lambda\|_1 \leq 1}w^TA\lambda = \|A^Tw\|_\infty$. It remains to show that $-A^Tw^k = \nabla L(\hat{\lambda}^k)$. To do so, first note that
\begin{equation*}
\nabla L(\hat{\lambda}^k)_j = \frac{\sum_{i = 1}^m-A_{ij}\exp(-(A\hat{\lambda}^k)_i)}{\sum_{\ell = 1}^m\exp(-(A\hat{\lambda}^k)_\ell)} \ .
\end{equation*}
Thus, defining a vector $\hat{w}^k \in \Delta_m$ by
\begin{equation*}
\hat{w}^k_i := \frac{\exp(-(A\hat{\lambda}^k)_i)}{\sum_{\ell = 1}^m\exp(-(A\hat{\lambda}^k)_\ell)} \ ,
\end{equation*}
then we have that $\nabla L(\hat{\lambda}^k) = -A^T\hat{w}^k$. Clearly, we have $\hat{w}^0 = w^0$. By way of induction, supposing that $\hat{w}^{k} = w^k$, then by the update in step (2.) of AdaBoost we have that
\begin{align*}
w^{k+1}_i &\propto w^k_i\exp(-\alpha_kA_{ij_k})\\
&= \hat{w}^k_i\exp(-\alpha_kA_{ij_k}) \\
&\propto \exp(-(A\hat{\lambda}^k)_i - \alpha_kA_{ij_k})\\
&= \exp(-(A(\hat{\lambda}^k + \alpha_ke_{j_k}))_i)\\
&= \exp(-(A\hat{\lambda}^{k+1})_i) \ .
\end{align*}
Since both $w^{k+1}$ and $\hat{w}^{k+1}$ are normalized, we have that $w^{k+1} = \hat{w}^{k+1}$. Therefore, we have that $-A^Tw^k = \nabla L(\hat{\lambda}^k)$ for all $k \geq 0$, and in particular $\|A^Tw^k\|_\infty = \|\nabla L(\hat{\lambda}^k)\|_\infty$.
\end{proof}

The equivalences given by Theorem \ref{adaBoost-equiv} and Lemma \ref{edge-lemma} imply computational complexity results for AdaBoost for both the margin $p(\lambda)$ and the gradient of the log-exponential loss function, for a variety of step-size rules via Theorem \ref{proxcomplexity}, as follows.\medskip

\begin{theorem}\label{adaboost-complexity}{\bf (Complexity of AdaBoost)}
For all $k \geq 1$, the sequence of classifiers $\{H_k\}$, with coefficient vectors $\{\hat{\lambda}^k\}$, and their normalizations $\{\bar{H}_k\}$, with coefficient vectors $\{\lambda^k\}$, produced by AdaBoost satisfy:
\begin{equation}\label{ada-ineq1}
\min\limits_{i \in \{0,\ldots,k-1\}}\|\nabla L(\hat{\lambda}^i)\|_\infty - p(\lambda^k) \leq \frac{\ln(m) + \frac{1}{2}\sum_{i = 0}^{k-1}\alpha_i^2}{\sum_{i=0}^{k-1}\alpha_i} \ .
\end{equation}
If we decide a priori to run AdaBoost for $k \geq 1$ iterations and use a constant step-size $\alpha_i := \sqrt{\frac{2\ln(m)}{k}}$ for all $i =0, \ldots, k-1$, then:
\begin{equation}\label{ada-ineq2}
\min\limits_{i \in \{0,\ldots,k-1\}}\|\nabla L(\hat{\lambda}^i)\|_\infty - p(\lambda^k) \leq \sqrt{\frac{2\ln(m)}{k}} \ .
\end{equation}
If instead we use the dynamic step-size $\alpha_i := \sqrt{\frac{2\ln(m)}{i+1}}$, then:
\begin{equation}\label{ada-ineq3}
\min\limits_{i \in \{0,\ldots,k-1\}}\|\nabla L(\hat{\lambda}^i)\|_\infty - p(\lambda^k) \leq \frac{\sqrt{\frac{\ln(m)}{2}}\left[2 + \ln(k)\right]}{2(\sqrt{k + 1} - 1)} \ .
\end{equation}
\end{theorem}

\begin{proof}
By weak duality and invoking Lemma \ref{edge-lemma} we have $p(\lambda^k) \leq \rho^* \leq \min\limits_{i \in \{0,\ldots,k-1\}}f(w^i)=\min\limits_{i \in \{0,\ldots,k-1\}} \|\nabla L(\hat{\lambda}^i)\|_\infty$. By Lemma \ref{lipschitz-f}, we have that $f(\cdot)$ has Lipschitz function values with Lipschitz constant $L_f = \|A\|_{1, \infty} = 1$, and by Lemma \ref{entropy-max}, we have that $\max\limits_{w \in \Delta_m}D(w, w^0) = \ln(m)$. Thus \eqref{ada-ineq1} follows directly from \eqref{mirror_bound2} in Theorem \ref{proxcomplexity}. The bounds \eqref{ada-ineq2} and \eqref{ada-ineq3} follow from \eqref{opt-alpha-bound} and \eqref{dynamic-bound}, respectively.
\end{proof}

Let us now discuss these results.  Let $\rho^* := \max\limits_{\lambda \in \Delta_n} p(\lambda)$ be the maximum margin over all normalized classifiers. Since we are assuming that the set of base classifiers $\mathcal{H}$ is closed under negation, it is always the case that $\rho^\ast \geq 0$.  When $\rho^\ast > 0$, there is a vector $\lambda^\ast \in \Delta_n$ with $A\lambda^\ast > 0$, and thus the classifier determined by $\lambda^\ast$ separates the data. In this separable case, it is both intuitively and theoretically desirable \cite{boostMargin97} to find a classifier with high margin, i.e., one that is close to the optimal value $\rho^\ast$.  For any $k \geq 1$, by weak duality, we have that $\rho^\ast \leq \min\limits_{i \in \{0,\ldots,k-1\}}\|\nabla L(\hat{\lambda}^i)\|_\infty$, whereby the bounds in \eqref{ada-ineq1}, \eqref{ada-ineq2}, and \eqref{ada-ineq3} hold for $\rho^* - p(\lambda^k)$, and thus provide exact computational guarantees that bound the optimality gap $\rho^\ast - p(\lambda^k)$ of the classifier $\bar{H}_k$ produced by AdaBoost.\medskip

When $\rho^\ast = 0$, then the data is not separable, and achieving the maximum margin is trivial; for example the classifier $\frac{1}{2}h_1 + \frac{1}{2}(-h_1)$ achieves the optimal margin.  In this case the log-exponential loss function $L(\cdot)$ is a metric of algorithm performance.  For any $k \geq 1$, by weak duality, we have that $0 = \rho^\ast \geq p(\lambda^k)$, whereby the bounds in \eqref{ada-ineq1}, \eqref{ada-ineq2}, and \eqref{ada-ineq3} hold for $\min\limits_{i \in \{0,\ldots,k-1\}}\|\nabla L(\hat{\lambda}^i)\|_\infty$ - 0 and hence provide exact computational complexity bounds for the $\ell_\infty$ norm of the gradient of $L(\cdot)$ thereby guaranteeing the extent to which the classifiers $H_k$ (equivalently $\hat\lambda^k$) produced by AdaBoost satisfy the first-order optimality condition for minimizing $L(\cdot)$.\medskip

\section{$\FSe$ as Subgradient Descent}\label{FSeSubgrad}
Here we consider the linear regression model $\by = \bX\beta + \be$, with given response vector $\by \in \mathbb{R}^n$, given model matrix $\bX \in \mathbb{R}^{n \times p}$, regression coefficients $\beta \in \mathbb{R}^p$ and errors $\be \in \mathbb{R}^n$.
In the high-dimensional statistical regime, especially with $p \gg n$, a sparse linear model with few non-zero coefficients is often desirable.  In this context, $\ell_1$-penalized regression, i.e., LASSO \cite{Tibshirani1996LASSO}, is often used to perform variable selection and shrinkage in the coefficients and is known to yield models with good predictive performance. The Incremental Forward Stagewise algorithm ($\FSe$)~\cite{HastieFSe, ESLBook} with shrinkage factor $\varepsilon$ is a type of boosting algorithm for the linear regression problem. $\FSe$ generates a coefficient profile\footnote{A coefficient profile is a path of coefficients $\{\beta(\alpha)\}_{\alpha \in \B{\alpha}}$ where $\alpha$ parameterizes the path. In the context of $\FSe$, $\alpha$ indexes the $\ell_1$ arc-length of the coefficients.} by repeatedly updating (by a small amount $\varepsilon$) the coefficient of the variable most correlated with the current residuals.  A complete description of $\FSe$ is presented in Algorithm \ref{forward}.\medskip

\floatname{algorithm}{Algorithm}
\begin{algorithm}
\caption{Incremental Forward Stagewise algorithm ($\FSe$)}\label{forward}
\begin{algorithmic}
\STATE Initialize at $r^0 = \by$, $\beta^0 = 0, k = 0$\medskip

At iteration $k$:
\STATE 1. Compute:
\begin{description}
\item $j_k \in \argmax\limits_{j \in \{1, \ldots, p\}} |(r^k)^T\bX_j|$
\end{description}
\STATE 2. Set:
\begin{description}
\item $r^{k+1} \gets r^k - \varepsilon \ \sgn((r^k)^T\bX_{j_k})\bX_{j_k}$
\item $\beta^{k+1}_{j_k} \gets \beta^{k}_{j_k} + \varepsilon\ \sgn((r^k)^T\bX_{j_k})$
\item $\beta^{k+1}_j \gets \beta^{k}_j \ , j \neq j_k$
\end{description}
\end{algorithmic}
\end{algorithm}
\floatname{algorithm}{Method}
As a consequence of the update scheme in Step (2.) of Algorithm \ref{forward}, $\FSe$ has the following sparsity property:
\begin{equation}\label{sparsity1}
\|\beta^k\|_1 \le k\varepsilon\;\;\;  \mathrm{and~~~} \|\beta^k\|_0 \le k \ .
\end{equation}
Different choices of $\varepsilon$ lead to different instances; for example a choice of $\varepsilon_k := | (r^k)^T\bX_{j_k} |$ yields the Forward Stagewise algorithm (FS)~\cite{ESLBook}, which is a greedy version of best-subset selection.\medskip

\subsection{$\FSe$ is a Specific Case of Subgradient Descent}
We now show that $\FSe$ is in fact an instance of the subgradient descent method (algorithm \ref{subgraddescent}) to minimize the largest correlation between the residuals and the predictors, over the space of residuals.  Indeed, consider the convex optimization problem:
\begin{equation}\label{FS-problem}
\min\limits_{r \in P_{\text{res}} } \;\;  f(r) := \|\bX^Tr\|_{\infty}
\end{equation}
where $P_{\text{res}}:= \{r \in \mathbb{R}^n : r = \by - \bX\beta \text{ for some } \beta \in \mathbb{R}^p\}$ is the the space of residuals.  One can also interpret the value of the objective function $f(r)$ in \eqref{FS-problem} as measuring the $\ell_\infty$ norm of the gradient of the least-squares loss function $L(\beta) := \frac{1}{2}\|\by - \bX\beta\|_2^2$ at some (possibly non-unique) point $\beta \in \mathbb{R}^p$. We establish the following equivalence.\medskip

\begin{theorem}\label{FSequiv}
The $\FSe$ algorithm is an instance of the subgradient descent method to solve problem (\ref{FS-problem}), initialized at $r^0 = \by$ and with a constant step-size of $\varepsilon$ at each iteration.
\end{theorem}

\begin{proof}
$f(r)$ measures the maximum (over all columns $j \in \{1, \ldots, p\}$) absolute value of the correlation between $\bX_j$ and $r$, and so $f(\cdot)$ has the following representation:
\begin{equation}\label{f_rep}
f(r) := \|\bX^Tr\|_{\infty} = \max_{j \in \{1, \ldots, p\}}|r^T\bX_j| = \max_{\beta \in B_1} \ r^T\bX\beta \ ,
\end{equation}
thus by \eqref{subdiff} for any $r \in \mathbb{R}^n$ we have:
\begin{equation}\label{subgradfact1}
j^* \in \argmax_{j \in \{1, \ldots, p\}}|r^T\bX_j| \Longleftrightarrow \sgn(r^T\bX_{j^*})\bX_{j^*} \in \partial f(r) \ .
\end{equation}
It therefore follows that Step (1.) of $\FSe$ is identifying a vector $g^k := \sgn((r^k)^T\bX_{j_k})\bX_{j_k} \in \partial f(r^k)$. Furthermore, Step (2.) of $\FSe$ is taking a subgradient step with step-size $\varepsilon$, namely $r^{k+1} := r^k - \varepsilon g^k$.  By an easy induction, the iterates of $\FSe$ satisfy $r^k = \by - \bX\beta^k \in P_{\text{res}}$ whereby $r^{k+1} := r^k - \varepsilon g^k = \Pi_{P_{\text{res}}}(r^k - \varepsilon g^k)$.
\end{proof}

As with the Mirror Descent interpretation of AdaBoost, we use the subgradient descent interpretation of $\FSe$ to obtain computational guarantees for a variety of step-size sequences.\medskip

\begin{theorem}\label{forward-complexity} {\bf (Complexity of $\FSe$)} Let $\beta_{LS} \in \arg\min_{\beta}\|\by - \bX\beta\|_2^2$ be any least-squares solution of the regression model.  With the constant shrinkage factor $\varepsilon$, for any $k \geq 0$ it holds that:
\begin{equation}\label{forward-complexity-eqn0}
\min_{i \in \{0,\ldots,k\}}\|\bX^Tr^i\|_\infty \leq \frac{\|\bX\beta_{LS}\|_2^2}{2\varepsilon(k+1)}+ \frac{\varepsilon\|\bX\|_{1,2}^2}{2} \ .
\end{equation}
If we a priori decide to run $\FSe$ for $k$ iterations and set $\varepsilon := \frac{\|\bX\beta_{LS}\|_2}{\|\bX\|_{1,2}\sqrt{k+1}}$ then
\begin{equation}\label{forward-complexity-eqn}
\min_{i \in \{0,\ldots,k\}}\|\bX^Tr^i\|_\infty \leq  \frac{\|\bX\|_{1,2}\|\bX\beta_{LS}\|_2}{\sqrt{k+1}} \ .
\end{equation}
If instead the shrinkage factor is dynamically chosen as $\varepsilon=\varepsilon_k:= \frac{|(r^k)^T\bX_{j_k}|}{\|\bX_{j_k}\|_2^2}$ (this is the Forward Stagewise algorithm (FS)~\cite{ESLBook}), then the bound (\ref{forward-complexity-eqn}) holds for all values of $k$ without having to set $k$ a priori.
\end{theorem}

\begin{proof} Let $r_{LS}:=\by -\bX\beta_{LS}$ be the residuals of the least-squares solution $\bX\beta_{LS}$, and it follows from orthogonality that $\bX^T r_{LS}=0$ if and only if $\beta_{LS}$ is a least-squares solution, hence the optimal objective function value of \eqref{FS-problem} is $f^*=f(r_{LS})=0$ and $r^*:=r_{LS}$ is an optimal solution of \eqref{FS-problem}.  As subgradient descent is simply Mirror Descent using the Euclidean prox function $d(r)=\frac{1}{2}r^Tr $ on the space of the residuals $r \in P_{\text{res}}$, we apply Theorem \ref{proxcomplexity} with $r=r^*=r_{LS}$.  We have:
\begin{equation*}
D(r^*, r^0) = D(r_{LS}, r^0) = \frac{1}{2}\|r_{LS} - r^0\|_2^2 = \frac{1}{2}\|r_{LS} - \by\|_2^2 = \frac{1}{2}\|\bX\beta_{LS}\|_2^2 \ .
\end{equation*}
By Lemma \ref{lipschitz-f}, $f(\cdot)$ has Lipschitz function values (with respect to the $\ell_2$ norm) with Lipschitz constant $L_f = \|\bX\|_{1,2} = \max\limits_{j \in \{1, \ldots, p\}}\|\bX_j\|_2$. Using these facts and $f(r_{LS}) = f^\ast = 0$, inequality \eqref{mirror_bound1} in Theorem \ref{proxcomplexity} implies \eqref{forward-complexity-eqn0}.  Setting $\varepsilon := \frac{\|\bX\beta_{LS}\|_2}{\|\bX\|_{1,2}\sqrt{k+1}}$ and substituting into \eqref{forward-complexity-eqn0} yields \eqref{forward-complexity-eqn}. Finally, the step-size $\varepsilon_k:= \frac{|(r^k)^T\bX_{j_k}|}{\|\bX_{j_k}\|_2^2}$ is just the step-size used to yield \eqref{subgrad_bound2}, and in this context $L_f = \|\bX\|_{1,2}$ and $\|r^0-r^*\|_2=\|\bX\beta_{LS}\|_2$ from which \eqref{forward-complexity-eqn} follows again.
\end{proof}

The computational complexity bounds in Theorem \ref{forward-complexity} are of a similar spirit to those implied by Theorem \ref{adaboost-complexity}, and can be interpreted as a guarantee on the ``closeness" of the coefficient vectors $\{\beta^k\}$ to satisfying the classical optimality condition $\|\bX^Tr\|_\infty = 0 $ for the (unconstrained) least-squares minimization problem.\medskip

Note that in the high-dimensional regime with $p > n$ and $\rank(\bX) = n$, we have that $\by = \bX\beta_{LS}$, thus the selection of $\varepsilon$ to obtain \eqref{forward-complexity-eqn} does not require knowing (or computing) $\beta_{LS}$. Furthermore, we can always bound $\|\bX\beta_{LS}\|_2 \leq \|\by\|_2$ and choose $\varepsilon$ optimally with respect to the resulting bound in \eqref{forward-complexity-eqn0}.  The interest in the interpretation given by Theorem \ref{FSequiv} and the consequent complexity results in Theorem \ref{forward-complexity} is due to the sparsity and regularization properties \eqref{sparsity1} combined with the computational complexity, in contrast to $\beta_{LS}$ which is not guaranteed to have any such sparsity or regularization properties.  Indeed, due to Theorem \ref{forward-complexity}, $\FSe$ now has the specific advantage of balancing the sparsity and regularization properties \eqref{sparsity1} and the complexity guarantees given by Theorem \ref{forward-complexity} through the choices of the shrinkage parameter $\varepsilon$ and the number of iterations $k$.\medskip

\appendix

\section{Appendix}\label{proofs}
\begin{lemma}\label{lipschitz-f}
Suppose that $f(\cdot) : P \to \mathbb{R}$ is defined by $f(x) := \max\limits_{\lambda \in Q} \ x^TA\lambda$. Then $f(\cdot)$ has Lipschitz function values with Lipschitz constant $L_f := \max\limits_{\lambda \in Q}\|A\lambda\|_\ast$. In particular, if $Q \subseteq B^\sharp := \{\lambda : \|\lambda\|_{\sharp} \leq 1\}$ for some norm $\|\cdot\|_\sharp$, then $L_f \leq \|A\|_{\sharp, \ast}$ where $\|A\|_{\sharp, \ast}$ is the operator norm of $A$.
\end{lemma}
\begin{proof}
Let $x, x^\prime \in P$ and let $\tilde{\lambda} \in \argmax\limits_{\lambda \in Q} \ x^TA\lambda \ , \ \tilde{\lambda}^\prime \in \argmax\limits_{\lambda \in Q} \ (x^\prime)^TA\lambda$. Then
\begin{align*}
f(x) - f(x^\prime) &= x^TA\tilde{\lambda} - (x^\prime)^TA\tilde{\lambda}^\prime\\
&\leq x^TA\tilde{\lambda} - (x^\prime)^TA\tilde{\lambda}\\
&= (x - x^\prime)^TA\tilde{\lambda}\\
&\leq \|A\tilde{\lambda}\|_\ast\|x - x^\prime\|\\
&\leq L_f\|x - x^\prime\| \ ,
\end{align*}
and symmetrically we have $f(x^\prime) - f(x) \leq L_f\|x^\prime - x\|$. Clearly if $Q \subseteq B^\sharp$, then
\begin{equation*}
L_f = \max\limits_{\lambda \in Q}\|A\lambda\|_\ast \leq \max\limits_{\lambda \in B^\sharp}\|A\lambda\|_\ast = \|A\|_{\sharp, \ast} \ .
\end{equation*}
\end{proof}

\begin{lemma}\label{entropy-max}
Let $e(\cdot) : \Delta_n \to \mathbb{R}$ be the entropy function, defined by $e(x) = \sum_{i = 1}^nx_i\ln(x_i) + \ln(n)$, with induced Bregman distance $D(\cdot, \cdot)$, and let $w^0 = (1/n, \ldots, 1/n)$. Then, we have $\max\limits_{w \in \Delta_n}D(w,w^0) = \ln(n)$.
\end{lemma}
\begin{proof}
Clearly $e(w^0) = \ln(1/n) + \ln(n) = 0$ and since $\nabla e(w^0)_i = 1 + \ln(1/n) = 1 - \ln(n)$, we have for any $w \in \Delta_n$:
\begin{equation*}
\nabla e(w^0)^T(w - w^0) = (1 - \ln(n))\sum_{i = 1}^n(w_i - 1/n) = (1 - \ln(n))(1 - 1) = 0 \ .
\end{equation*}
Thus we have:
\begin{equation*}
D(w,w^0) = e(w) - e(w^0) - \nabla e(w^0)^T(w - w^0) = e(w) = \sum_{i = 1}^nw_i\ln(w_i) + \ln(n) \leq \ln(n) \ .
\end{equation*}
Furthermore, the maximum is achieved by $e_1 = (1, 0, \ldots, 0)$.
\end{proof}

\newpage

\bibliographystyle{amsplain}
\bibliography{GF-papers-orc_student_paper}

\end{document}